\documentclass{article}

\newif\ifcomm
\commfalse
\usepackage[disable]{todonotes}
\usepackage{pdfsync}
\usepackage{graphicx} 

\newcommand{\todoy}[2][]{\todo[color=red, #1]{#2}}

\usepackage[round,comma]{natbib}
\usepackage{algorithm}
\usepackage{algorithmic}

\usepackage{amsthm}

\usepackage{amssymb}
\usepackage{amsmath} 
\usepackage{color}
\usepackage{graphicx}
\usepackage{epstopdf}
\usepackage{algorithm,algorithmic}
\usepackage{verbatim} 

\usepackage{nicefrac}

\newif\ifcomm
\commtrue

\newif\iflong
\longtrue

\newtheorem{thm}{Theorem}
\newtheorem{cor}[thm]{Corollary}
\newtheorem{lem}[thm]{Lemma}

\theoremstyle{remark}
\newtheorem{rem}[thm]{Remark}

\newtheorem{remark}[thm]{Remark}

\newcounter{assumption}
\renewcommand{\theassumption}{A\arabic{assumption}}

\newenvironment{ass}[1][]{\begin{trivlist}\item[] \refstepcounter{assumption}%
 {\bf Assumption\ \theassumption\ #1} }{
 \ifvmode\smallskip\fi\end{trivlist}}

\newcommand{\norm}[1]{\left\Vert#1\right\Vert}

\newcommand{\abs}[1]{\left\vert#1\right\vert}

\newcommand{\trace}{\mathop{\rm trace}}

\newcommand{\set}[1]{\left\{#1\right\}}

\newcommand{\Real}{\mathbb R}                        
\newcommand{\real}{\mathbb R}                        

\newcommand{\Prob}[1]{{\mathbb P}\left(#1\right)}    
\newcommand{\EE}[1]{{\mathbb E}\left[#1\right]}      
\newcommand{\E}{{\mathbb E}}                         

\newcommand{\MC}{\mathcal{C}}

\newcommand{\MD}{\mathcal{D}}

\newcommand{\eps}{\varepsilon}                       

\newcommand{\ra}{\rightarrow}

\newcommand{\argmax}{\mathop{\rm argmax}}

\newcommand{\ip}[2]{\langle #1,#2 \rangle}
\newcommand{\eqdef}{\stackrel{\mbox{\rm\tiny def}}{=}}

\newcommand{\beq}{\begin{equation}}
\newcommand{\eeq}{\end{equation}}
\newcommand{\beqa}{\begin{eqnarray}}
\newcommand{\eeqa}{\end{eqnarray}}
\newcommand{\beqan}{\begin{eqnarray*}}
\newcommand{\eeqan}{\end{eqnarray*}}
\newcommand{\ben}{\begin{eqnarray*}}
\newcommand{\een}{\end{eqnarray*}}
\newcommand{\bea}{\begin{align*}}
\newcommand{\eea}{\end{align*}}

\ifcomm
   \newcommand\comm[1]{\textcolor{blue}{ #1}}
   \newcommand{\mtodo}[2]{\todo{{\bf #1}: #2}} 
   \def\here#1{{\bf $\langle\langle$#1$\rangle\rangle$}}

\else
   \newcommand\comm[1]{}
   \newcommand{\mtodo}[2]{}
   \def\here#1{}
\fi

\newcommand{\sfrac}{\nicefrac}

\renewcommand{\phi}{\varphi}

\newcommand{\e}{\mathbf{e}}
\renewcommand{\eps}{\varepsilon}

\newcommand{\ttop}{^\top}
\newcommand{\Stexpra}{\sum_{k=1}^{t} \eta_k m_{k-1}}

\newcommand{\Vinit}{V}

\newcommand{\CPE}[2]{\EE{ #1 \,| #2 }}

\newcommand{\normm}[2]{\norm{#1}_{#2}}
\newcommand{\Vta}[1]{\overline{V}_{#1}}
\newcommand{\Vt}{\Vta{t}}
\newcommand{\Vtta}[1]{V_{#1}}
\newcommand{\Vtt}{\Vtta{t}}
\newcommand{\beps}{\mathbb{\varepsilon}}

\newcommand{\FF}{{\cal F}}

\newcommand{\hth}{\hat{\theta}}

\usepackage{colt11e}

\newif\ifadmin
\adminfalse
 \admintrue 

\title{Online Least Squares Estimation with Self-Normalized Processes: An Application to Bandit Problems\thanks{Submitted to the 24th Annual Conference on Learning Theory (COLT 2011)}}

\author{
Yasin Abbasi-Yadkori\\
\texttt{\small abbasiya@cs.ualberta.ca}\\
Dept. of Computing Science\\
University of Alberta 
\And
D\'avid P\'al\\
\texttt{\small dpal@cs.ualberta.ca} \\
Dept. of Computing Science \\
University of Alberta 
\And
Csaba Szepesv\'ari\\
\texttt{\small szepesva@cs.ualberta.ca }\\
Dept. of Computing Science \\
University of Alberta 
}

\begin{document}

\maketitle

\begin{abstract}
The analysis of online least squares estimation is at the heart of many
stochastic sequential decision-making problems.  We employ tools from the
self-normalized processes to provide a simple and self-contained proof of a
tail bound of a vector-valued martingale.  We use the bound to construct new
tighter confidence sets for the least squares estimate. 

We apply the confidence sets to several online decision problems, such as the
multi-armed and the linearly parametrized bandit problems.  The confidence sets
are potentially applicable to other problems such as sleeping bandits,
generalized linear bandits, and other linear control problems. 

We improve the regret bound of the Upper Confidence Bound (UCB) algorithm
of~\citet{AuerBandit} and show that its regret is with high-probability a
problem dependent constant. In the case of linear bandits~\citep{DaniStoch}, we
improve the problem dependent bound in the dimension and number of time steps.
Furthermore, as opposed to the previous result, we prove that our bound holds
for small sample sizes, and at the same time the worst case bound is improved
by a logarithmic factor and the constant is improved. 
\end{abstract}

\section{Introduction}

The least squares method forms a cornerstone of statistics and machine
learning. It is used as the main component of many stochastic sequential
decision problems, such as multi-armed bandit, linear bandits, and other linear
control problems. However, the analysis of least squares in these online
settings is non-trivial because of the correlations between data points.
Fortunately, there is a connection between online least squares estimation and
the area of self-normalized processes. Study of self-normalized processes has a
long history that goes back to Student and is treated in detail in recent
book by~\citet{delaPenaLaiShao:09}. Using these tools we provide a proof of a
bound on the deviation for vector-valued martingales. A
less general version of the bound can be found already in~\cite{DeLaPenaKlaLai04,
delaPenaLaiShao:09}. Additionally our proof, based on the method of mixtures, is new, simpler and self-contained.
The bound improves the previous bound of~\cite{RusTsi10} and it is applicable
to virtually any online least squares problem.

The bound that we derive, gives immediately rise to tight confidence sets for
the online least squares estimate that can replace the confidence sets in existing
algorithms. In particular, the confidence sets can be used in the UCB algorithm for the
multi-armed bandit problem, the \textsc{ConfidenceBall} algorithm
of~\citet{DaniStoch} for the linear bandit problem, and \textsc{LinRel}
algorithm of~\cite{Auer02:JMLR} for the associative reinforcement learning
problem. We show that this leads to improved performance of these algorithms.
Our hope is that the new confidence sets can be used to improve the performance
of other similar linear decision problems.

The multi-armed bandit problem, introduced by \citet{ro52}, is a game between
the learner and the environment. At each time step, the learner chooses one of
$K$ actions and receives a reward which is generated independently at random
from a fixed distribution associated with the chosen arm. The objective of the
learner is to maximize his total reward.  The performance of the learner is
evaluated by the regret, which is defined as the difference between his total
reward and the total reward of the best action. \citet{LaiRo85} prove a
$(\sum_{i\neq i_*}1/D(p_j,p_{i_*})-o(1))\log T$ lower bound on the expected
regret of any algorithm, where $T$ is the number of time steps, $p_{i_*}$ and
$p_i$ are the reward distributions of the optimal arm and arm $i$ respectively,
and $D$ is the KL-divergence.

\citet{AuerBandit} designed the UCB algorithm and proved a finite-time
logarithmic bound on its regret. He used Hoeffding's inequality to construct
confidence intervals and obtained a $O((K \log T)/ \Delta)$ bound on the
expected regret, where $\Delta$ is the difference between the expected rewards
of the best and the second best action.\todoy{Cite high-probability bound for
UCB by Bubeck.}~ We modify UCB so that it uses our new confidence sets and we
show a stronger result. Namely, we show that with probability $1-\delta$, the
regret of the modified algorithm is $O(K \log(1/\delta)/\Delta)$. Seemingly,
this result contradicts the lower bound of \citet{LaiRo85}, however our
algorithm depends on $\delta$ which it receives as an input. The expected
regret of the modified algorithm with $\delta = 1/T$ matches the regret of 
the original algorithm. 

In the linear bandit problem, the learner chooses repeatedly actions from a
fixed subset of $\Real^d$ and receives a random reward, expectation of which is
a linear function of the action. \citet{DaniStoch} proposed the
\textsc{ConfidenceBall} algorithm and showed that its regret is at most $O(d
\log(T) \sqrt{T \log(T/\delta)})$ with probability at most $1-\delta$. We
modify their algorithm so that it uses our new confidence sets and we show that
its regret is at most $O(d\log(T)\sqrt{T} + \sqrt{dT\log(T/\delta)})$.
Additionally, constants in our bound are smaller, and our bound holds for all
$T \ge 1$, as opposed the previous one which holds only for sufficiently large
$T$.  \citet{DaniStoch} prove also a problem dependent regret bound. Namely,
they show that the regret of their algorithm is $O(\frac{d^2}{\Delta} \log^2 T
\log(T/\delta))$ where $\Delta$ is the ``gap'' as defined in~\citep{DaniStoch}.
For our modified algorithm we prove an improved
$O(\frac{\log(1/\delta)}{\Delta}(\log T + d \log \log T)^2)$ bound.

\subsection{Notation}
We use $\|\cdot\|$ to denote the 2-norm.
For a positive definite matrix $A\in \real^{d\times d}$, the weighted $2$-norm is defined by $\|x\|^2_A = x\ttop A x$, where $x\in \real^d$. The inner product is denoted by $\ip{\cdot}{\cdot}$ and the weighted inner-product $x\ttop A y = \ip{x}{y}_{A}$.
We use $\lambda_{\min}(A)$ to denote the minimum eigenvalue of the positive definite matrix $A$.
We use $A\succ 0$ to denote that $A$ is positive  definite, while we use $A\succeq 0$ to denote that it is positive semidefinite. The same notation is used to denote the Loewner partial order of matrices.
We shall use  $\e_i$ to denote the $i^{\rm th}$ unit vector, i.e., for all $j\neq i $, $\e_{ij}=0$ and $\e_{ii}=1$.


\section{Vector-Valued Martingale Tail Inequalities}
Let $(\FF_k;k\ge 0)$ be a filtration, $(m_k;k\ge 0)$ be an $\real^d$-valued stochastic process adapted to $(\FF_k)$, $(\eta_k;k\ge 1)$ be a real-valued martingale difference process adapted to $(\FF_k)$.
Assume that $\eta_k$ is conditionally sub-Gaussian in the sense that there exists some $R>0$ such that for any $\gamma\in \real$, $k\ge 1$,
\begin{equation}\label{eq:subgauss}
\E[ \exp(\gamma \eta_k ) \,|\, \FF_{k-1} ] \le \exp\left( \frac{\gamma^2 R^2}{2}\right) \quad \mathrm{a.s.}
\end{equation}
Consider the martingale 
\beq\label{eq:st}
S_t  = \Stexpra
\eeq
and the matrix-valued processes 
\beq\label{eq:vt}
\Vtt = \sum_{k=1}^t m_{k-1}m_{k-1}\ttop, \qquad \Vt  = \Vinit + \Vtt, \quad t\ge 0,
\eeq
where $\Vinit $ is an $\FF_{0}$-measurable, positive definite matrix. In particular, assume that with probability one, the eigenvalues of $\Vinit $ are larger than $\lambda_0>0$ and that $\|m_k\|\le L$ holds a.s. for any $k\ge 0$.

The following standard inequality plays a crucial role in the following developments:
\begin{lem}\label{lem:basicineq}
Consider $(\eta_t)$, $(m_t)$ as defined above and let $\tau$ be a stopping time with respect to the filtration $(\FF_t)$.
Let $\lambda\in \real^d$ be arbitrary and consider
\[
P_t^\lambda=\exp\left( \sum_{k=1}^t\,\, \left[\frac{ \eta_k \ip{\lambda}{m_{k-1}}}{R} - \frac12 \, \ip{\lambda}{m_{k-1}}^2 \right]\right).
\]
Then $P_\tau$ is almost surely well-defined and
\[
\EE{ P_\tau^\lambda } \le 1.
\]
\end{lem}
\begin{proof}
The proof is standard (and is given only for the sake of completeness).
We claim that $P_t =P_t^\lambda$ is a supermartingale.
Let
\[
D_k = \exp\left( \frac{ \eta_k \ip{\lambda}{m_{k-1}}}{R} - \frac12 \, \ip{\lambda}{m_{k-1}}^2 \right).
\]
Observe that by~\eqref{eq:subgauss}, we have $\CPE{ D_k }{\FF_{k-1}} \le 1$. 
 Clearly, $D_k$ is $\FF_k$-adapted, as is $P_k$.
Further,
\begin{align*}
 \EE{ P_t | \FF_{t-1} }  
 &= \CPE{ D_1 \cdots  D_{t-1}  D_t }{\FF_{t-1} }
   =  D_1 \cdots D_{t-1}\, \CPE{ D_t }{\FF_{t-1}} 
 \le  P_{t-1},
\end{align*}
showing that $(P_t)$ is indeed a supermartingale.

Now, this immediately leads to the desired result when $\tau=t$ for some deterministic time $t$.
This is based on the fact that the mean of any supermartingale can be bounded by the mean of its first element. In the case of $(P_t)$, for example, we have
$\EE{ P_t } =\EE{ \EE{ P_t|\FF_{t-1} } }\le \EE{ P_{t-1} } \le \ldots \le \EE{P_0} =\EE{D_0} = 1$.

Now, in order to consider the general case, let $S_t = P_{\tau\wedge t}$.\footnote{$\tau\wedge t$ is a shorthand notation for $\min(\tau,t)$.}
It is well known that $(S_t)$ is still a supermartingale with $\EE{S_t}\le \EE{S_0}=\EE{P_0}= 1$.
Further, since $P_t$ was nonnegative, so is $S_t$.
Hence, by the convergence theorem for nonnegative supermartingales, 
 almost surely, $\lim_{t\ra\infty} S_t$ exists, i.e., $P_\tau$ is almost surely well-defined.
Further, $\EE{ P_\tau } = \EE{ \liminf_{t\ra\infty} S_t } \le \liminf_{t\ra\infty} \EE{ S_t } \le 1$ by Fatou's Lemma.
\end{proof}
Before stating our main results, we give some recent results, which can essentially be extracted from the paper by \cite{RusTsi10}.
\begin{thm}
\label{thm:vvmartingaletail2}
Consider the processes $(S_t )$, $(\Vt )$ as defined above and let 
\begin{equation}\label{eq:kappadef}
\kappa = \sqrt{3+2\log((L^2+\trace(V))/\lambda_0)}.
\end{equation}
Then, for any $0<\delta< 1$, $t\ge 2$, with probability at least $1-\delta$,
\begin{equation}\label{eq:xitnormbound}
\| S_t  \|_{\Vt^{-1}} 
	   \le  2\, \kappa^2 R  \sqrt{ \log t} \,\sqrt{d\,\log(t) +\log(1/\delta)}\,.
\end{equation}
\end{thm}
\iflong
We now show how to strengthen the previous result using the method of mixtures, originally used by \citet{RobSie70} to evaluate boundary crossing probabilities for Brownian motion.
\begin{thm}[Self-normalized bound for vector-valued martingales]
\label{thm:detbound}
Let $(\eta_t)$, $(m_t)$, $(S_t)$, $(\Vt)$, and $(\FF_t)$ be as before and let $\tau$ be a stopping time
 with respect to the filtration $(\FF_t)$.
Assume that $\Vinit$ is deterministic.
Then, for any $0<\delta<1$,   with probability $1-\delta$,
\beq\label{eq:keybound}
\newcommand{\keybound}{
2 R^2 \log\left( \frac{\det (\Vta{\tau})^{\sfrac12}\det(\Vinit)^{\sfrac{\kern-2pt-\kern-2pt 1}{2}}}{\delta} \right)
}
\normm{S_\tau}{\Vta{\tau}^{-1}}^2 
\le
\keybound\,.
\eeq
\end{thm}
\begin{proof}
Without loss of generality, assume that $R=1$ (by appropriately scaling $S_t$, this can always be achieved).
Let 
\begin{align*}
M_t(\lambda) 
 &= \exp\left( \, \ip{\lambda}{S_t} - \tfrac12 \,\normm{\lambda}{\Vtta{t}}^2 \,\right)\,.
\end{align*}
Notice that by Lemma~\ref{lem:basicineq}, the mean of $M_\tau(\lambda)$
is not larger than one.

Let $\Lambda$ be a Gaussian random variable which is independent of all the other random variables and whose covariance is $\Vinit^{-1}$.
Define
\[
M_t = \EE{ M_t(\Lambda) | \FF_\infty}.
\]
Clearly, we still have $\EE{M_\tau} =\EE{\,\EE{ \,M_\tau(\Lambda) \,| \,\Lambda\, }\,} \le 1$.

Let us calculate $M_t$:
Let $f$ denote the density of $\Lambda$ and for a positive definite matrix $P$ let
$c(P) = \sqrt{(2\pi)^d/\det (P)}=\int \exp(-\tfrac12 x\ttop P x ) dx $.
Then,
\begin{align*}
M_t 
  &= \int_{\real^d} 
  			\exp\left( \ip{\lambda}{S_t}-\tfrac12\,\normm{\lambda}{\Vtt}^2 \right)\, 
			f(\lambda) \,d\lambda \\
  &=  \int_{\real^d}
   			\exp\left( -\tfrac12\, \normm{\lambda - \Vtt^{-1} S_t}{\Vtt}^2 + \tfrac12\, \normm{ S_t}{\Vtt^{-1}}^2 \right)\,
   			 f(\lambda) \,d\lambda  \\
  &= \frac{1}{c(\Vinit)}\, \exp\left( \tfrac12\, \normm{ S_t}{\Vtt^{-1}}^2 \right)
        \,\int_{\real^d} 
   			\exp\left( -\tfrac12\, \left\{ \normm{\lambda - \Vtt^{-1} S_t}{\Vtt}^2 
						                         + \normm{\lambda}{\Vinit}^2\right\} \right) \,d\lambda.
\end{align*}
Elementary calculation shows that if $P\succeq 0$, $Q\succ 0$,
\[
\normm{x-a}{P}^2 + \normm{x}{Q}^2 =
\normm{ x - (P+Q)^{-1} Pa }{P+Q}^2 + \normm{a}{P}^2 - \normm{Pa}{(P+Q)^{-1}}^2.
\]
Therefore, 
\begin{align*}
\normm{\lambda - \Vtt^{-1} S_t}{\Vtt}^2 + \normm{\lambda}{\Vinit}^2
\quad & =\quad
  \normm{ \lambda - (\Vinit+\Vtt)^{-1}  S_t }{\Vinit+\Vtt}^2 
+ \normm{\Vtt^{-1} S_t}{\Vtt}^2 - \normm{S_t}{(\Vinit+\Vtt)^{-1}}^2\\
& =\quad 
  \normm{ \lambda - (\Vinit+\Vtt)^{-1}  S_t }{\Vinit+\Vtt}^2 
+ \normm{ S_t}{\Vtt^{-1}}^2 - \normm{S_t}{(\Vinit+\Vtt)^{-1}}^2,
\end{align*}
which gives
\begin{align*}
M_t 
  &= \frac{1}{c(\Vinit)}\, \exp\left( \tfrac12\, \normm{S_t}{(\Vinit+\Vtt)^{-1}}^2\right)
        \,\int_{\real^d} 
        \exp\left( -\tfrac12\,
			  \normm{ \lambda - (\Vinit+\Vtt)^{-1}  S_t }{\Vinit+\Vtt}^2  \right)
        \, d\lambda \\
  &= \frac{c(\Vinit+\Vtt)}{c(\Vinit)}
  		\, \exp\left( \tfrac12\, \normm{S_t}{(\Vinit+\Vtt)^{-1}}^2\right)
    = \left( \frac{\det (\Vinit)}{\det (\Vinit+\Vtt)} \right)^{1/2} 
    		\, \exp\left( \tfrac12\, \normm{S_t}{(\Vinit+\Vtt)^{-1}}^2\right)\,.
\end{align*}
Now, from $\EE{M_\tau}\le 1$, we obtain
\begin{align*}
\Prob{\normm{S_\tau}{(\Vinit+\Vtta{\tau})^{-1}}^2 
>
2 \log\left( \,\frac{\det (\Vinit+\Vtta{\tau})^{\sfrac12}}{\det(\Vinit)^{\sfrac12}} \, \,\frac{1}{\delta}\right)}
& =
\Prob{
\frac{\exp\left(\,\tfrac12\,\normm{S_\tau}{(\Vinit+\Vtta{\tau})^{-1}}^2 \,\right)}
	   {\delta^{-1} \left( \lower -.7ex\hbox{$\det (\Vinit+\Vtta{\tau})$} \Big/\lower .7ex\hbox{$\det(\Vinit)$}\right)^{\tfrac12} }
> 1 }
 \\
& \le 
\EE{\frac{\exp\left(\,\tfrac12\,\normm{S_\tau}{(\Vinit+\Vtta{\tau})^{-1}}^2 \,\right)}
	   {\delta^{-1} \left( \lower -.7ex\hbox{$\det (\Vinit+\Vtta{\tau})$} \Big/\lower .7ex\hbox{$\det(\Vinit)$}\right)^{\tfrac12} }
}\\
& = \EE{M_\tau}\delta \le \delta,
\end{align*}
thus finishing the proof.
\end{proof}

\begin{corollary}[Uniform Bound]
\label{cor:union}
Under the same assumptions as in the previous theorem,
for any $0<\delta<1$, with probability $1-\delta$, 
\begin{equation}
\forall t \ge 0, \qquad \normm{S_t}{\Vta{t}^{-1}}^2 
\le
2 R^2 \log\left( \frac{\det (\Vta{t})^{\sfrac12}\det(\Vinit)^{\sfrac{\kern-2pt-\kern-2pt 1}{2}}}{\delta} \right).
\end{equation}
\end{corollary}

\begin{proof}
We will use a stopping time construction, which goes back at least to \cite{Fre75}. Define the bad event
\begin{equation}
\label{eq:Eq1}
B_t(\delta) = \left\{ \omega \in \Omega ~:~ 
\norm{S_t}_{\bar{V}_t^{-1}}^2 > 2R^2 \log \left( \frac{\det(\bar{V}_t)^{1/2} \det(V)^{-1/2}}{\delta} \right)
\right\}
\end{equation}
We are interested in bounding the probability that $\bigcup_{t \ge 0} B_t(\delta)$ happens. 
Define $\tau(\omega)=\min\{ t \geq 0 ~:~ \omega\in B_t(\delta) \}$, with the convention that $\min \emptyset=\infty$. Then, $\tau$ is a stopping time. Further, 
$$\bigcup_{t \ge 0} B_t(\delta)=\{ \omega ~:~ \tau(\omega) < \infty\}.$$
Thus, by Theorem~\ref{thm:detbound}
\begin{align*}
\Prob{\bigcup_{t \geq 0} B_t(\delta)} &= \Prob{\tau < \infty}\\
&= \Prob{\norm{S_{\tau}}_{\bar{V}_{\tau}^{-1}}^2 > 2R^2 \log \left( \frac{\det(\bar{V}_{\tau})^{1/2} \det(V)^{-1/2}}{\delta} \right),\, \tau<\infty} \\
&\leq \Prob{\norm{S_{\tau}}_{\bar{V}_{\tau}^{-1}}^2 > 2R^2 \log \left( \frac{\det(\bar{V}_{\tau})^{1/2} \det(V)^{-1/2}}{\delta} \right)}\\
&\leq \delta \; .
\end{align*}
\end{proof}

Let us now turn our attention to understanding the determinant term on the right-hand side of~\eqref{eq:keybound}.

\begin{lem}
\label{lem:proj}
We have that
$$\log \frac{\det( \Vt )}{\det V}  \le  \sum_{k=1}^t \normm{m_{k-1}}{\Vta{k-1}^{-1}}^2.$$
Further, we have that
\begin{align*}
 \sum_{k=1}^t \left(\normm{m_{k-1}}{\Vta{k-1}^{-1}}^2 \wedge 1\right) &\le 2(\log \det(\Vt) - \log \det \Vinit) \le 2(d \log((\trace(V)+t L^2)/d) - \log \det \Vinit).
\end{align*}
Finally, if $\lambda_0 \ge \max(1,L^2)$ then 
$$\sum_{k=1}^t \normm{m_{k-1}}{\Vta{k-1}^{-1}}^2 \le 
2 \log \frac{\det( \Vt )}{\det(V)}.$$
\end{lem}

\begin{proof}
Elementary algebra gives
\begin{align}
\det(\Vt) &= \det(\Vta{t-1} + m_{t-1}m_{t-1}\ttop ) = \det(\Vta{t-1}) \det( I+ \Vta{t-1}^{-\sfrac12} m_{t-1} (\Vta{t-1}^{-\sfrac12} m_{t-1})\ttop ) \nonumber \\
& = \det(\Vta{t-1})\, (1+\normm{m_{t-1}}{\Vta{t-1}^{-1}}^2 ) = \det(\Vinit) \prod_{k=1}^t\left(1+\normm{m_{k-1}}{\Vta{k-1}^{-1}}^2\right), \label{eq:detvt}
\end{align}
where we used that all the eigenvalues of a matrix of the form $I+xx\ttop$ are one except one eigenvalue, which is $1+\norm{x}^2$ and which corresponds to the eigenvector $x$. Using $\log(1+t)\le t$, we can bound $\log\det(\Vt)$ by
\[
\log\det(\Vt) \le \log\det V + \sum_{k=1}^t \normm{m_{k-1}}{\Vta{k-1}^{-1}}^2.
\]
Combining $x\le 2\log(1+x)$, which holds when $x\in[0,1]$, and~\eqref{eq:detvt}, we get
\begin{align*}
 \sum_{k=1}^t \left(\normm{m_{k-1}}{\Vta{k-1}^{-1}}^2 \wedge 1\right)
  \le 2  \sum_{k=1}^t \log \left( 1+ \normm{m_{k-1}}{\Vta{k-1}^{-1}}^2 \right)
  = 2(\log \det(\Vt) - \log \det \Vinit).
\end{align*}
The trace of $\Vt$ is bounded by $\trace(V)+ t L^2$, assuming $\|m_k\| \le L$.
Hence, 
 $\det(\Vt) = \prod_{i=1}^d \lambda_i \le \left( \frac{\trace(V)+t L^2}{d} \right)^d$ and therefore,
\[
\log \det(\Vt) \le d \log((\trace(V)+t L^2)/d),
\]
finishing the proof of the second inequality. The sum $ \sum_{k=1}^t \normm{m_{k-1}}{\Vta{k-1}^{-1}}^2$ can itself be upper bounded as a function of $\log \det(\Vt)$ provided that $\lambda_0$ is large enough.
Notice $\normm{m_{k-1}}{\Vta{k-1}^{-1}}^2 \le \lambda_{\min}^{-1}(\Vta{k-1}) \norm{m_{k-1}}^2 \le L^2/\lambda_0$. Hence, we get that if $\lambda_0 \ge \max(1,L^2)$,
$$\log \frac{\det( \Vt )}{\det V}  \le  \sum_{k=1}^t \normm{m_{k-1}}{\Vta{k-1}^{-1}}^2 \le 
2 \log \frac{\det( \Vt )}{\det(V)}.$$
\end{proof}

Most of this argument can be extracted from the paper of~\cite{DaniStoch}.
However, the idea goes back at least to \citet{LaiRobWei79,LaiWei82} (a similar argument is used around Theorem~11.7 in the book by \citet{CBLu06:book}).
Note that Lemmas B.9--B.11 of \citet{RusTsi10} also give a bound on  $\sum_{k=1}^t \normm{m_{k-1}}{\Vta{k-1}^{-1}}^2$, with an essentially identical argument. 
Alternatively, one can use the bounding technique of~\citet{Auer02:JMLR} (see the proof of Lemma~13 there on pages 412--413) to derive a bound like $\sum_{k=1}^t \normm{m_{k-1}}{\Vta{k-1}^{-1}}^2 \le C d \log t$ for a suitable chosen constant $C>0$.

\begin{rem}
By combining Corollary~\ref{cor:union} and Lemma~\ref{lem:proj}, we get a simple worst case bound that holds with probability $1-\delta$:
\beq
\label{eq:worstcase}
\forall t\geq 0,\quad \normm{S_t}{\Vt^{-1}}^2 
\le
 d\, R^2 \log\left( \frac{\trace(V)+tL^2}{d\,\delta}\right).
\eeq
Still, the new bound is considerably better than the previous one given by Theorem~\ref{thm:vvmartingaletail2}. 
Note that the $\log(t)$ factor cannot be removed, as shown by Problem 3, page 203 in the book by \citet{delaPenaLaiShao:09}. 
\end{rem}

\section{Optional Skipping}

Consider the case when $d=1$, $m_{k}=\eps_k\in \{0,1\}$, i.e., the case of an optional skipping process. Then, using again $\Vinit=I=1$, $\Vt =  1+\sum_{k=1}^t \eps_{k-1}\eqdef 1+N_t$ and thus
the expression studied becomes
$$
\normm{S_t}{\Vt^{-1}} = \frac{|\sum_{k=1}^t \eps_{k-1}\eta_k |}{\sqrt{1+N_t}}.
$$
We also have 
$$
\log \det(\Vt) = \sum_{k=1}^t \log\left( 1+ \frac{\eps_{k-1}}{1+N_k} \right) \le \sum_{k=1}^t \frac{\eps_{k-1}}{1+N_k} 
=\sum_{k=1}^{N_t+1} \frac{1}{k} \le 1+ \int_1^{N_t+1} x^{-1} \,dx = 1+\log(1+N_t).
$$
Thus, we get, with probability $1-\delta$
\begin{equation}
\label{eq:confIntr}
\forall s \geq 0, \quad
\left|\sum_{k=1}^s \eps_{k-1} \eta_k \right|
\le \sqrt{ (1+N_s)\,  \left(1+2\log\left(\frac{(1+N_s)^{\sfrac{1}{2}}}{\delta}\right)\right) }\,.
\end{equation}
If we apply Doob's optional skipping and Hoeffding-Azuma, with a union bound (see, e.g., the paper of \citet{BuMuStSz08:NIPS}), we would get,
for any $0<\delta<1$, $t\ge 2$, with probability $1-\delta$,
\begin{equation}
\label{eq:optskipunionbound}
\forall s \in \{0,\dots,t\}, \qquad
\left|\sum_{k=1}^s \eps_{k-1} \eta_k \right|
\le \sqrt{2N_s\, \log\left(\frac{2t}{\delta}\right)}.
\end{equation}
The major difference between these bounds is that~\eqref{eq:optskipunionbound}
depends explicitly on $t$, while~\eqref{eq:confIntr} does not. This has the
positive effect that one need not recompute the bound if $N_t$ does not grow,
which helps e.g. in the paper of \cite{BuMuStSz08:NIPS} to improve the
computational complexity of the HOO algorithm. Also, the coefficient of the
leading term in~\eqref{eq:confIntr} under the square root is $1$, whereas
in~\eqref{eq:optskipunionbound} it is $2$.

Instead of a union bound, it is possible to use a ``peeling device'' to replace the conservative $\log t$ factor in the above bound by essentially $\log \log t$. This is done e.g. in \citet{Garivier:2008p816} in their Theorem~22.\footnote{They give their theorem as ratios, which they should not, since their inequality then fails to hold for $N_t=0$. However, this is easy to remedy by reformulating their result as we do it here.}
From their derivations, the following one sided, uniform bound can be extracted (see
Remark~24, page~19): For any $0<\delta<1$, $t\ge 2$, with probability $1-\delta$,
\begin{equation}
\label{eq:uniform}
\forall s \in \{0,\ldots,t\}, \qquad
\sum_{k=1}^s \eps_{k-1} \eta_k \le \sqrt{ \frac{4 \,N_s}{1.99} \, \log\left( \frac{6 \log t}{\delta}  \right)}.
\end{equation}
As noted by \citet{Garivier:2008p816}, due to the law of iterated logarithm, the scaling of the right-hand side as a function of $t$ cannot be improved in the worst-case. However, this leaves open the possibility of deriving a maximal inequality which depends on $t$ only through $N_t$.

\section{The Multi-Armed Bandit Problem}
\label{sec:finiteBandit}

Now we turn our attention to the multi-armed bandit problem. Let $\mu_i$ denote
the expected reward of action $i$ and $\Delta_i=\mu_*-\mu_i$, where $\mu_*$ is
the expected reward of the optimal action. We assume that if we choose action $I_t$
in round $t$, we obtain reward $\mu_{I_t} + \eta_t$.  
Let $N_{i,t}$ denote the number of times that
we have played action $i$ up to time $t$, and $\bar{X}_{i,t}$ denote the
average of the rewards received by action $i$ up to time $t$.
From~\eqref{eq:confIntr} with $\delta/K$ instead of $\delta$ and a union bound over the actions, 
we have the following confidence intervals that hold
with probability at least $1-\delta$:
\begin{equation}
\label{eq:confBndFinite}
\forall i\in\{1,\dots,K\}, \  \forall s\in\{1,2,\dots\}, \qquad \abs{\bar{X}_{i,s}-\mu_i} \le c_{i,s} \; ,
\end{equation}
where
$$c_{i,s}=\sqrt{ \frac{1+N_{i,s}}{N_{i,s}^2}\,  \left(1+2\log\left(\frac{K(1+N_{i,s})^{\sfrac{1}{2}}}{\delta}\right)\right) }.$$
Modify the UCB Algorithm of~\citet{AuerBandit} to use the confidence intervals~\eqref{eq:confBndFinite} and change the action selection rule accordingly. Hence, at time $t$, we choose the action
\begin{equation}
\label{eq:actSelFinite}
I_t=\argmax_i \, \bar{X}_{i,t} + c_{i,t}.
\end{equation}
We call this algorithm UCB($\delta$).

\begin{thm}
\label{thm:UCBImp}
With probability at least $1-\delta$, the total regret of the UCB($\delta$) algorithm with the action selection rule~\eqref{eq:actSelFinite} is constant and is bounded by 
$$R(T) \leq \sum_{i: \Delta_i >0} \left( 3 \Delta_i + \frac{16}{\Delta_i} \log \frac{2K}{\Delta_i \delta} \right).$$
where $i_*$ is the index of the optimal action. 
\end{thm}
\begin{proof}
Suppose the confidence intervals do not fail. If we play action $i$, the upper estimate of the action is above $\mu^*$. Hence,
$$c_{i,s} \geq \frac{\Delta_i}{2}.$$
Substituting $c_{i,s}$ and squaring gives
$$\frac{N_{i,s}^2-1}{N_{i,s}+1}\leq \frac{N_{i,s}^2}{N_{i,s}+1}\leq \frac{4}{\Delta_i^2}\left( 1 + 2 \log \frac{K (1+N_{i,s})^{1/2}}{\delta} \right).$$
By using Lemma 8 of~\citet{AntosActive}, we get that
$$N_{i,s} \leq 3 + \frac{16}{\Delta_i^2} \log \frac{2K}{\Delta_i \delta}.$$
Thus, using $R(T)=\sum_{i\neq i_*} \Delta_i N_{i,T}$, we get that with probability at least $1-\delta$, the total regret is bounded by 
$$R(T) \leq \sum_{i: \Delta_i >0} \left( 3 \Delta_i + \frac{16}{\Delta_i} \log \frac{2K}{\Delta_i \delta} \right).$$
\end{proof}
\begin{remark}
\citet{LaiRo85} prove that for any suboptimal arm $j$, 
$$\EE{N_{i,t}}\geq \frac{\log t}{D(p_j,p_*)},$$
where, $p_*$ and $p_j$ are the reward density of the optimal arm and arm $j$ respectively, and $D$ is the KD-divergence. This lower bound does not contradict Theorem~\ref{thm:UCBImp}, as Theorem~\ref{thm:UCBImp} only states a high probability upper bound for the regret. Note that UCB($\delta$) takes delta as its input. Because with probability $\delta$, the regret in time $t$ can be $t$, on expectation, the algorithm might have a regret of $t \delta$.
Now if we select $\delta = 1/t$, then we get $O(\log t)$ upper bound on the expected regret.
\end{remark}


\section{Application to Least Squares Estimation and Linear Bandit Problem}
\label{sec:linBandit}

In this section we first apply Theorem~\ref{thm:detbound} to derive confidence intervals for least-squares estimation, where the covariate process is an arbitrary process and then use these confidence intervals to improve the regret bound of~\citet{DaniStoch} for the linear bandit problem. In particular, our assumption on the data is as follows:
\begin{ass}\label{ass:linearresponse}
Let $(\FF_i)$ be a filtration,
 $(x_1,y_1)$, $\ldots$, $(x_{t},y_t)$ be a sequence of random variables 
 over $\real^d\times \real$ such that 
$x_i$ is $\FF_i$-measurable, 
and $y_i$ is $\FF_{i+1}$-measurable $(i=1,2,\ldots)$.
Assume that there exists $\theta_*\in \real^d$ such that $\EE{y_i|\FF_i} = x_i\ttop \theta_*$, i.e.,
 $\eps_i = y_i - x_i\ttop \theta_*$ is a martingale difference sequence
 ($\EE{\eps_i| \FF_i} = 0$, $i=1,2,\ldots$) 
 and that $\eps_i$ is sub-Gaussian: There exists $R>0$ such that for any $\gamma \in \real$,
\[
\EE{ \exp( \gamma \eps_i ) | \FF_{i-1} } \le \exp( \gamma^2 R^2/2).
\]
\end{ass}
We shall call the random variables $x_i$ covariates and the random variables $y_i$ the responses.
Note that the assumption allows any sequential generation of the covariates.

Let $\hth_t$ be the $\ell^2$-regularized least-squares estimate of $\theta_*$ with regularization parameter $\lambda>0$:
\beq\label{eq:lsest}
\hth_t = (X\ttop X + \lambda I )^{-1} X\ttop Y, \qquad \hth_0 = 0,
\eeq
where $X$ is the matrix whose rows are $x_1^\top,\ldots,x_{t-1}^\top$ and $Y = (y_1,\ldots,y_{t-1})^\top$.
We further let $\beps = (\eps_1,\ldots,\eps_{t-1})^\top$.

We are interested in deriving a confidence bound on the error of predicting the mean response $x\ttop \theta_*$ at an arbitrarily chosen random covariate $x$ using the least-squares predictor $x\ttop \hth_t$.
Using 
\beqan
\hth_t 
 &=& (X\ttop X + \lambda I )^{-1} X \ttop (X \theta_* + \beps) \\
 &=& (X\ttop X + \lambda I )^{-1}X\ttop  \beps
 		 +  (X\ttop X + \lambda I )^{-1} (X \ttop X+\lambda I) \theta_*
       - \lambda (X\ttop X + \lambda I )^{-1}  \theta_*\\
 &=& (X\ttop X + \lambda I )^{-1} X\ttop \beps + \theta_* 
       - \lambda (X\ttop X + \lambda I )^{-1}  \theta_*\,,
\eeqan
we get
\begin{align*}
x\ttop \hth_t - x\ttop \theta_* 
  & = x\ttop (X\ttop X + \lambda I )^{-1} X\ttop  \beps
  - \lambda x\ttop (X\ttop X + \lambda I )^{-1}  \theta_* \\
  & = \ip{x}{X\ttop  \beps}_{V_t^{-1}} - \lambda \ip{x}{\theta_*}_{V_t^{-1}},
\end{align*}
where $V_t = X\ttop X + \lambda I $.
Note that $V_t$ is positive definite (thanks to $\lambda>0$) and hence so is $V_t^{-1}$, so the above inner product is well-defined. Using the Cauchy-Schwartz inequality, we get
\beqan
|x\ttop \hth_t - x\ttop \theta_*|
 &\le
  \normm{x}{V_t^{-1}} \left( \normm{ X\ttop  \beps }{V_t^{-1}} + \lambda \, \normm{\theta_*}{V_t^{-1}} \right) \\
 &\le
  \normm{x}{V_t^{-1}} \left( \normm{ X\ttop  \beps }{V_t^{-1}} + \lambda^{1/2}\,\norm{\theta_*} \right),
\eeqan
where we used that $\normm{\theta_*}{V_t^{-1}}^2 \le 1/\lambda_{\min}(V_t) \norm{\theta_*}^2 \le 1/\lambda \norm{\theta_*}^2$.
Fix any $0<\delta<1$. By Corollary~\ref{cor:union},
 with probability at least $1-\delta$,
 \newcommand{\keyboundappliedp}[1]{
 R \sqrt{2 \log\left( \frac{\det (V_{#1})^{\sfrac12}\det(\lambda I)^{\sfrac{\kern-2pt-\kern-2pt 1}{2}}}{\delta} \right)}
}
\newcommand{\keyboundapplied}{\keyboundappliedp{t}}
 \newcommand{\keyboundappliedworstp}[1]{
 R \sqrt{d \log\left( \frac{1+\frac{#1 L}{\lambda}}{\delta} \right)}
}
 \newcommand{\keyboundappliedworst}{\keyboundappliedworstp{t}}
\[
\forall t\geq 1, \quad \normm{ X\ttop  \beps }{V_t^{-1}} \le \keyboundapplied.
\]
Therefore, on the event where this inequality holds, one also has
\[
|x\ttop \hth_t - x\ttop \theta_*|
 \le
  \normm{x}{V_t^{-1}} \left( \keyboundapplied + \lambda^{1/2}\,\norm{\theta_*}\right).
\]
Similarly, we can derive a worst-case bound.
The result is summarized in the following statement:
\begin{thm}\label{eq:lsthm}
Let $(x_1,y_1),\ldots,(x_{t-1},y_{t-1})$, $x_i\in \real^d$, $y_i\in \real$ satisfy the linear model Assumption~\ref{ass:linearresponse} with some $R>0$, $\theta_*\in \real^d$ and let  $(\FF_t)$ be the associated filtration.
Assume that w.p.1 the covariates satisfy $\norm{x_i}\le L$, $i=1,\ldots,n$
and $\norm{\theta_*}\le S$.
Consider the $\ell^2$-regularized least-squares parameter estimate $\hth_n$ with regularization coefficient $\lambda>0$ (cf.~\eqref{eq:lsest}). 
Let $x$ be an arbitary, $\real^d$-valued random variable.
Let $V_t= \lambda I + \sum_{i=1}^{t-1} x_i x_i\ttop$ be the regularized design matrix underlying the covariates.
Then, for any $0<\delta< 1$, with probability at least $1-\delta$,
\beq\label{eq:lsbound}
\forall t\geq 1, \quad |x\ttop \hth_{t} - x\ttop \theta_*|
 \le
  \normm{x}{V_{t}^{-1}} \left( \keyboundappliedp{t} + \lambda^{1/2}\, S\right).
\eeq
Similarly, with probability $1-\delta$,
\beq
\forall t\geq 1,\quad |x\ttop \hth_{t} - x\ttop \theta_*|
\le
  \normm{x}{V_{t}^{-1}} \left( \keyboundappliedworstp{t}+ \lambda^{1/2}\, S\right).
\eeq
\end{thm}

\begin{remark}
We see that $\lambda\ra \infty$ increases the second term (the ``bias term'') in the parenthesis of the estimate. In fact, $\lambda\ra\infty$ for $n$ fixed gives $\lambda^{1/2} \normm{x}{V_t^{-1}} \ra {\rm const}$ (as it should be).
Decreasing $\lambda$, on the other hand increases $\normm{x}{V_t^{-1}}$ and the $\log$ term, while it decreases the bias term $\lambda^{1/2}S$.
\end{remark}

From the above result, we immediately obtain confidence bounds for $\theta_*$: 
\begin{cor}\label{eq:confellipsbound}
Under the condition of Theorem~\ref{eq:lsthm}, with probability at least $1-\delta$,
\[
\forall t\geq 1,\quad \normm{\hth_{t} -  \theta_*}{V_{t}}
 \le
   \keyboundappliedp{t} + \lambda^{1/2}\, S.
\]
Also,  with probability at least $1-\delta$,
\[
\forall t\geq 1,\quad \normm{\hth_{t} -  \theta_*}{V_{t}}
 \le
   \keyboundappliedworstp{t} + \lambda^{1/2}\, S.
\]
\end{cor}
\begin{proof}
Plugging in  $x = V_{t}(\hth_{t} - \theta_*)$ into~\eqref{eq:lsbound}, we get
\beq\label{eq:cs1}
\normm{\hth_{t} -  \theta_*}{V_{t}}^2
 \le
  \normm{V_{t}(\hth_{t} - \theta_*)}{V_{t}^{-1}} \left( \keyboundappliedp{t} + \lambda^{1/2}\, S\right).
\eeq
Now, $\normm{V_{t}(\hth_{t} - \theta_*)}{V_{t}^{-1}}^2 = \normm{\hth_{t} -  \theta_*}{V_{t}}^2$ and therefore 
either $\normm{\hth_{t} -  \theta_*}{V_{t}}=0$, in which case the conclusion holds, or we can divide both sides of~\eqref{eq:cs1} by $\normm{\hth_{t} -  \theta_*}{V_{t}}$ to obtain the desired result.
\end{proof}

\begin{remark}
In fact, 
 the theorem and the corollary are equivalent.
To see this note that $x\ttop (\hth_{t}-\theta_*) = (\hth_{t} - \theta_*)\ttop V_{t}^{1/2} V_{t}^{-1/2} x$,
thus
\[
\sup_{x\not=0} \frac{ |x\ttop (\hth_{t}-\theta_*)| }{  \phantom{\mbox{}_{V}}\normm{x}{V_{t}^{-1}} } = \normm{\hth_{t}-\theta_*}{V_{t}}.
\]
\end{remark}
\begin{remark}
The above bound could be compared with a similar bound of \citet{DaniStoch} whose bound, under identical conditions, states that (with appropriate initialization) 
with probability $1-\delta$,
\beq\label{eq:danibound}
\mbox{for all}\,\, t\,\,\mbox{large enough,}\qquad
\normm{\hth_t -  \theta_*}{V_t}
\le R \max\left\{ \sqrt{128 \,d \log(t) \, \log\left(\frac{t^2}{\delta}\right)}, 
   							\frac83 \, \log\left(\frac{t^2}{\delta}\right) 
				\right\}\,,
\eeq
where large enough means that $t$ satisfies $0<\delta<t^2 e^{-1/16}$. Denote by $\sqrt{\beta_t(\delta)}$ the right-hand side in the above bound. The restriction on $t$ comes from the fact that $\beta_t(\delta) \ge 2d (1+2 \log(t))$ is needed in the proof of the last inequality of their Theorem~5. 

On the other hand, Theorem~\ref{thm:vvmartingaletail2} gives rise to the following result:
For any {\em fixed} $t\ge 2$, for any $0<\delta<1$, with probability at least $1-\delta$,
\[
\normm{\hth_t -  \theta_*}{V_t}
\le 2\, \kappa^2 R  \sqrt{ \log t} \,\sqrt{d\,\log(t) +\log(1/\delta)} + \lambda^{1/2} S\,,
\]
where $\kappa$ is as in Theorem~\ref{thm:vvmartingaletail2}.
To get a uniform bound one can use a union bound with $\delta_t = \delta/t^2$. Then $\sum_{t=2}^\infty \delta_t =
\delta (\tfrac{\pi^2}{6} - 1) \le \delta$. This thus gives that for any $0<\delta<1$, with probability at least $1-\delta$,
\[
\mbox{for all}\quad t\ge 2, \quad
\normm{\hth_t -  \theta_*}{V_t}
\le 2\, \kappa^2 R  \sqrt{ \log t} \,\sqrt{d\,\log(t) +\log(t^2/\delta)} + \lambda^{1/2} S\,,
\]
This looks tighter than~\eqref{eq:danibound}, but is still lagging beyond the result of Corollary~\ref{eq:confellipsbound}.
\end{remark}

\subsection{The Linear Bandit Problem}
We now turn our attention to the linear bandit problem. Assume the actions lie in $\MD\subset \real^d$ and for any $x\in \MD$, $\norm{x}^2\leq L$. Assume the reward of taking action $x\in \MD$ has the form of
$$h_t(x)=\theta_*\ttop x + \eta_t$$
and assume $\forall x\in \MD,\, \theta_*\ttop x \in [-1,1]$. Define the regret by 
$$R(T)=\sum_{t=1}^T (\theta_*\ttop x_*  - \theta_*\ttop x_t),$$ 
where $x_*$ is the optimal action ($x_* = \argmax_{x\in \MD} \theta_*\ttop x$). 
Define the confidence set
\beq
\label{eq:confSet}
\MC_t(\delta) = \set{\theta: (\theta-\hat{\theta}_t)\ttop V_t(\theta-\hat{\theta}_t)\leq \beta_t(\delta) },
\eeq
where
$$\beta_t(\delta) = \left( \keyboundappliedp{t} + \lambda^{1/2}\, S\right)^2.$$
Consider the \textsc{ConfidenceBall} algorithm of~\citet{DaniStoch}. We use the confidence intervals~\eqref{eq:confSet} and change the action selection rule accordingly. Hence, at time $t$, we define $\tilde{\theta}_t$ and $x_t$ by the following equation:
\beq
\label{eq:Eq2}
(\tilde{\theta}_t, x_t) = \argmax_{(\theta, x)\in \MC_t(\delta)\times \MD} \theta\ttop x.
\eeq
The algorithm is shown in Table~\ref{alg:LinBandit}.

\begin{table}
\begin{center}
\framebox{\parbox{12cm}{
\begin{algorithmic}
\STATE \textbf{Input}: Confidence $0<\delta<1$.
\FOR{$t:=1,2,\dots$}
\STATE $(\tilde{\theta}_t, x_t) = \argmax_{(\theta, x)\in \MC_t(\delta)\times \MD} \theta\ttop x$.
\STATE Play $x_t$ and observe reward $h_t(x_t)$.
\STATE Update $V_t$ and $C_t$.
\ENDFOR
\end{algorithmic}
}}
\end{center}
\caption{The Linear Bandit Algorithm}
\label{alg:LinBandit}
\end{table}

\begin{thm}
\label{thm:ConfEllpImp}
With probability at least $1-\delta$, the regret of the Linear Bandit Algorithm shown in Table~\ref{alg:LinBandit} satisfies
$$\forall T\geq 1,\quad R(T) \leq 4 \sqrt{T d \log (\lambda + T L /d)} \left( \lambda^{1/2} S + R \sqrt{2\log 1/\delta + d \log( 1+ T L/(\lambda d) )} \right).$$
\end{thm}

\begin{proof}

Lets decompose the instantaneous regret as follows:
\begin{align}
\label{eqn:regretDecomposed}
\notag
r_t &= \theta_*\ttop x_* - \theta_*\ttop x_t \\
\notag
&\leq \tilde{\theta}_t\ttop x_t - \theta_*\ttop x_t \\
\notag
&=(\tilde{\theta}_t-\theta_* )\ttop x_t \\
\notag
&= (\hat{\theta}_t-\theta_*)\ttop x_t + (\tilde{\theta}_t-\hat{\theta}_t)\ttop x_t \\
&\leq \sqrt{\beta_t(\delta)} \normm{x_t}{V_{t}^{-1}},
\end{align}
where the last step holds by Cauchy-Schwarz. 
Using~\eqref{eqn:regretDecomposed} and the fact that $r_t\leq 2$, we get that
$$r_t \leq 2\min(\sqrt{\beta_t(\delta)}\normm{x_t}{V_{t}^{-1}}^2,1) \leq 2\sqrt{\beta_t(\delta)}\min(\normm{x_t}{V_{t}^{-1}}^2,1).$$
Thus, with probability at least $1-\delta$, $\forall T\geq 1$
\begin{align*}
R(T) &\leq \sqrt{T \sum_{t=1}^T r_t^2} \leq  \sqrt{8\beta_{T}T \sum_{t=1}^{T} \min(w_t^2,1)} \leq 4 \sqrt{\beta_T T \log (\det(V_{T}))}\\
&\leq 4 \sqrt{T d \log (\lambda + t L /d)} \left( \lambda^{1/2} S + R \sqrt{2\log 1/\delta + d \log( 1+t L/(\lambda d) )} \right).
\end{align*}
where the last two steps follow from Lemma~\ref{lem:proj}.
\end{proof}



\subsection{Saving Computation}
\label{sec:effAlg}

The action selection rule~\eqref{eq:Eq2} is NP-hard in general~\citep{DaniStoch}. In this section, we show that we essentially need to solve this problem only $O(\log t)$ times up to time $t$ and hence saving computations. Algorithm~\ref{alg:LinBanditFixAct} achieves this objective by changing its policy only when the volume of the confidence set is halved and still enjoyes almost the same regret bound as for Algorithm~\ref{alg:LinBandit}.

\begin{table}
\begin{center}
\framebox{\parbox{12cm}{
\begin{algorithmic}
\STATE \textbf{Input}: Confidence $0<\delta<1$.
\STATE $\tau=1$ \COMMENT{This is the last timestep that we changed the action}
\FOR{$t:=1,2,\dots$}
\IF{$\det(V_t)>2\det(V_\tau)$}
\STATE $(\tilde{\theta}_t, x_t) = \argmax_{(\theta, x)\in \MC_t(\delta)\times \MD} \theta\ttop x$.
\STATE $\tau=t$.
\ENDIF
\STATE $x_t=x_\tau$.
\STATE Play $x_t$ and observe reward $h_t(x_t)$.
\ENDFOR
\end{algorithmic}
}}
\end{center}
\caption{The Linear Bandit Algorithm}
\label{alg:LinBanditFixAct}
\end{table}

\begin{thm}
\label{thm:ConfEllpEff}
With probability at least $1-\delta$, $\forall T\geq 1$, the regret of the Linear Bandit Algorithm shown in Table~\ref{alg:LinBanditFixAct} satisfies
\begin{align*}
R(T) \leq 4 \sqrt{2T d \log (\lambda + T L /d)} \left( \lambda^{1/2} S + R \sqrt{2\log 1/\delta + d \log( 1+T L/(\lambda d) )} \right)+ 4\sqrt{d\log(T/d)}.
\end{align*}
\end{thm}

First, we prove the following lemma:

\begin{lem}
\label{lem:vols}
Let $A$, $B$ and $C$ be positive semi-definite matrices such that $A=B+C$. Then, we have that
$$\sup_{x\neq 0} \frac{x\ttop A x}{x\ttop B x}\leq \frac{\det(A)}{\det(B)}.$$
\end{lem}

\begin{proof}

We consider first a simple case.
Let $A=B+m m\ttop$, $B$ positive definite. 
Let $x\neq 0$ be an arbitrary vector. Using the Cauchy-Schwartz inequality, we get
$$(x\ttop m)^2 = (x\ttop B^{1/2} B^{-1/2} m)^2 \leq \norm{B^{1/2} x}^2 \norm{B^{-1/2} m}^2 = \norm{x}_{B}^2 \norm{m}_{B^{-1}}^2.$$
Thus,
$$x\ttop (B + mm\ttop)x \leq x\ttop B x + \norm{x}_{B}^2 \norm{m}_{B^{-1}}^2=(1+\norm{m}_{B^{-1}}^2)\norm{x}_{B}^2$$
and so
$$\frac{x\ttop A x}{x\ttop B x}\leq 1+\norm{m}_{B^{-1}}^2.$$
We also have that 
\begin{align*}
\det(A) &= \det(B + m m\ttop) = \det(B) \det(I + B^{-1/2} m (B^{-1/2} m)\ttop) = \det(B) (1+ \norm{m}_{B^{-1}}^2),
\end{align*}
thus finishing the proof of this case.

If $A=B + m_1 m_1\ttop + \dots + m_{t-1} m_{t-1}\ttop,$ then define $V_{s}=B + m_1 m_1\ttop + \dots + m_{s-1} m_{s-1}\ttop$ and use 
$$\frac{x\ttop A x}{x\ttop B x} = \frac{x\ttop V_t x}{x\ttop V_{t-1} x}\frac{x\ttop V_{t-1} x}{x\ttop V_{t-2} x} \dots \frac{x\ttop V_2 x}{x\ttop B x}.$$
By the above argument, since all the terms are positive, we get
$$\frac{x\ttop A x}{x\ttop B x} \leq \frac{\det(V_t)}{\det(V_{t-1})}\frac{\det(V_{t-1})}{\det(V_{t-2})}\dots \frac{\det(V_2)}{\det(B)}=\frac{\det(V_t)}{\det(B)}=\frac{\det(A)}{\det(B)}.$$
This finishes the proof of this case. 

Now, if $C$ is a positive definite matrix, then the eigendecomposition of $C$ gives $C=U\ttop \Lambda U$ , where $U$ is orthonormal and $\Lambda$ is positive diagonal matrix. 
This, in fact gives that $C$ can be written as the sum of at most $d$ rank-one matrices, finishing the proof for the general case.

\end{proof}

\begin{proof} [Proof of Theorem~\ref{thm:ConfEllpEff}]
Let $\tau_t$ be the smallest timestep $\leq t$ such that $x_t=x_{\tau_t}$. By an argument similar to the one used in Theorem~\ref{thm:ConfEllpImp}, we have 
$$r_t \leq (\hat{\theta}_{\tau_t} - \theta_*)\ttop x_t + (\tilde{\theta}_{\tau_t} - \hat{\theta}_{\tau_t})\ttop x_t.$$
We also have that for all $\theta\in \MC_{\tau_t}$ and $x$,
\begin{align*}
\abs{(\theta - \hat{\theta}_{\tau_t})\ttop x} &\leq \norm{V_t^{1/2}(\theta - \hat{\theta}_\tau)}\sqrt{x\ttop V_t^{-1}x}\\
&\leq \norm{V_{\tau_t}^{1/2}(\theta - \hat{\theta}_{\tau_t})}\sqrt{\frac{\det(V_t)}{\det(V_{\tau_t})}}\sqrt{x\ttop V_t^{-1}x}\\
&\leq \sqrt{2}\norm{V_{\tau_t}^{1/2}(\theta - \hat{\theta}_\tau)}\sqrt{x\ttop V_t^{-1}x}\\
&\leq \sqrt{2 \beta_{\tau_t}}\sqrt{x\ttop V_t^{-1}x},
\end{align*} 
where the second step follows from Lemma~\ref{lem:vols}, and the third step follows from the fact that at time $t$ we have $\det(V_t)<2 \det(V_{\tau_t})$.
The rest of the argument is identical to that of Theorem~\ref{thm:ConfEllpImp}. We conclude that with probability at least $1-\delta$, $\forall T\geq 1$,
$$R(T) \leq 4 \sqrt{2T d \log (\lambda + t L /d)} \left( \lambda^{1/2} S + R \sqrt{2\log 1/\delta + d \log( 1+t L/(\lambda d) )} \right).$$
\end{proof}


\subsection{Problem Dependent Bound ($\Delta>0$)}

Let $\Delta$ be as defined in~\citep{DaniStoch}. In this section we assume that $\Delta>0$. This includes the case when the action set is a polytope. First we state a matrix perturbation theorem from~\citet{StewartSunBook} that will be used later.
\begin{thm}
[\citet{StewartSunBook}, Corollary~4.9]
\label{thm:Perturbation}
Let $A$ be a symmetric matrix with eigenvalues $\nu_1 \geq \nu_2 \geq \dotso \geq \nu_d$, $E$ be a symmetric matrix with eigenvalues $e_1\geq e_2 \geq \dotso \geq e_d$, and $V=A+E$ denote a symmetric perturbation of $A$ such that the eigenvalues of $V$ are $\tilde{\nu}_1 \geq \tilde{\nu}_2 \geq \dotso  \geq \tilde{\nu}_d$. Then, for $i=1,\dots,d$,
$$\tilde{\nu}_i\in [\nu_i+e_d, \nu_i+e_1].$$
\end{thm}

\begin{thm}
\label{thm:ConfEllpDeltaImp}
Assume that $\Delta>0$ for the gap $\Delta$ defined in~\citep{DaniStoch}. Further assume that $\lambda \geq 1$ and $S\geq 1$. With probability at least $1-\delta$, $\forall T\geq 1$, the regret of the algorithm shown in Table~\ref{alg:LinBandit} satisfies
$$R(T)=\frac{16 R^2 \lambda S^2}{\Delta} \left(\log (L T) + (d-1)\log \frac{64 R^2 \lambda S^2 L}{\Delta^2} + 2(d-1)\log \left( d\log \frac{d\lambda + T L^2}{d} + 2\log (1/\delta)\right) + 2\log (1/\delta)\right)^2.$$
\end{thm}

\begin{proof}
First we bound the regret in terms of $\log \det(V_T)$. We have that
\begin{align}
\label{eq:probDep}
R(T) &= \sum_{t=1}^T r_t \leq \sum_{t=1}^T \frac{r_t^2}{\Delta} \leq \frac{16\beta_T}{\Delta} \log(\det(V_T)),
\end{align}
where the first inequality follows from the fact that either $r_t =0$ or $\Delta < r_t$, and the second inequality can be extracted from the proof of Theorem~\ref{thm:ConfEllpImp}. 
Let $b_t$ be the number of times we have played a sub-optimal action (an action $x_s$ for which $\theta_*\ttop x_* - \theta_*\ttop x_s \geq \Delta$) up to time $t$. Next we bound $\log \det(V_t)$ in terms of $b_t$. 
We bound the eigenvalues of $V_t$ by using Theorem~\ref{thm:Perturbation}.

Let $E_t=\sum_{s: x_s\neq x_*}^t x_s x_s\ttop$ and $A_t=V_t-E_t=(t-b_t)x_* x_*\ttop$. The only non-zero eigenvalue of $(t-b_t)x_* x_*\ttop$ is $(t-b_t)L^*$, where $L^* = x_*\ttop x_* \leq L$. Let the eigenvalues of $V_t$ and $E_t$ be $\lambda_1\geq \dots \geq \lambda_d$ and $e_1\geq \dots \geq e_d$ respectively. By Theorem~\ref{thm:Perturbation}, we have that
$$\lambda_1 \in [(t-b_t)L^*+e_d, (t-b_t)L^*+e_1]\quad\mbox{and}\quad \forall i\in \{2,\dots,d\},\, \lambda_i \in [e_d, e_1].$$
Thus,
\begin{equation*}
\det(V_t) = \prod_i^d \lambda_i \leq ((t-b_t)L^*+e_1)e_1^{d-1} \leq ((t-b_t)L+e_1)e_1^{d-1}.
\end{equation*}
Therefore,
$$\log \det(V_t)\leq \log((t-b_t)L+e_1) + (d-1)\log e_1.$$
Because $\trace(E)=\sum_{s: x_s\neq x_*}^t \trace(x_s x_s\ttop)\leq L b_t$, we conclude that $e_1 \leq L b_t$. Thus,
\begin{align}
\label{eq:withbt}
\notag
\log \det(V_t)&\leq\log((t-b_t)L+L b_t) + (d-1)\log (L b_t)\\
&=\log(L t) + (d-1)\log (L b_t).
\end{align}
With some calculations, we can show that
\beq
\label{eq:t1}
\beta_t \log \det V_t \leq 4 R^2 \lambda S^2 (2\log (1/\delta) + \log \det V_t)^2 \leq 4 R^2 \lambda S^2\left(d\log \frac{d\lambda + t L^2}{d} +2 \log \frac{1}{\delta}\right)^2,
\eeq
where the second inequality follows from Lemma~\ref{lem:proj}. 
Hence,
\beq
\label{eq:t2}
b_t \leq \frac{16\beta_t}{\Delta^2} \log(\det(V_t)) \leq \frac{64 R^2 \lambda S^2}{\Delta^2}\left(d\log \frac{d\lambda + t L^2}{d} +2 \log \frac{1}{\delta}\right)^2,
\eeq
where the first inequality follows from $R(t)\geq b_t \Delta$.  
Thus, with probability $1-\delta$, $\forall T\geq 1$,
\begin{align*}
R(T) &\leq \frac{16\beta_T}{\Delta} \log(\det(V_T)) \\
& \leq \frac{64 R^2 \lambda S^2}{\Delta} (\log(\det(V_T)) + 2\log (1/\delta))^2\\
&\leq \frac{16 R^2 \lambda S^2}{\Delta} (\log (L T) +  (d-1)\log (Lb_T) + 2\log (1/\delta))^2\\
&\leq \frac{16 R^2 \lambda S^2}{\Delta} \left(\log (L T) + (d-1)\log \frac{64 R^2 \lambda S^2 L}{\Delta^2} + 2(d-1)\log \left( d\log \frac{d\lambda + T L^2}{d} + 2\log (1/\delta)\right) + 2\log (1/\delta)\right)^2,
\end{align*}
where the first step follows from~\eqref{eq:probDep}, the second step follows from the first inequality in~\eqref{eq:t1}, the third step follows from~\eqref{eq:withbt}, and the last step follows from the second inequality in~\eqref{eq:t2}. 
\end{proof}

\begin{rem}
The problem dependent regret of~\citep{DaniStoch} scales like $O(\frac{d^2}{\Delta} \log^3 T)$, while our bound scales like $O(\frac{1}{\Delta}(\log^2 T + d \log T + d^2 \log\log T))$.
\end{rem}

\bibliography{all_bib}

\end{document}